\documentclass[11pt,a4paper]{article}
\usepackage[utf8]{inputenc}
\usepackage[T1]{fontenc}
\usepackage[english]{babel}
\usepackage{microtype}
\usepackage[margin=1in]{geometry}
\usepackage{amsmath,amssymb,amsthm,mathtools}
\usepackage{graphicx}
\usepackage{booktabs}
\usepackage{multirow}
\usepackage{array}
\usepackage{algorithm}
\usepackage{algpseudocode}
\usepackage{xcolor}
\usepackage{tcolorbox}
\usepackage[hidelinks]{hyperref}
\usepackage{cleveref}
\usepackage{enumitem}
\usepackage{tikz}
\usetikzlibrary{arrows.meta, positioning, calc, shapes.geometric, fit, backgrounds}
\usepackage{orcidlink}

\tcbuselibrary{skins,breakable}

\hypersetup{
  colorlinks=true,
  linkcolor=blue!50!black,
  citecolor=blue!50!black,
  urlcolor=blue!50!black
}

\newcommand{\R}{\mathbb{R}}

\newcommand{\calS}{\mathcal{S}}
\newcommand{\calA}{\mathcal{A}}
\newcommand{\calG}{\mathcal{G}}
\newcommand{\calM}{\mathcal{M}}
\newcommand{\pihi}{\pi^{\mathrm{H}}}

\DeclareMathOperator{\Done}{Done}

\theoremstyle{plain}
\newtheorem{proposition}{Proposition}[section]
\theoremstyle{definition}
\newtheorem{definition}{Definition}[section]

\title{\vspace{-2em}\textbf{Policy-Invariant Reward Shaping from LLM Feedback: A Framework for Hybrid RL Agents}}

\author{
  Christophe D.\ Hounwanou$^{1,2}$\thanks{Corresponding author: \texttt{christophe.hounwanou@aims.ac.rw}}
    \,\orcidlink{0009-0006-2181-0952} \quad
  John Emeka Eze$^{1}$
    \,\orcidlink{0009-0008-0688-5125} \quad
  Ya\'e U. Gaba$^{2,3,4}$
    \,\orcidlink{0000-0001-8128-9704} \\[0.6em]
  {\small $^{1}$African Institute for Mathematical Sciences, Rwanda} \\
  {\small $^{2}$AI Research and Innovation Nexus for Africa (AIRINA Labs), AI.Technipreneurs, Bénin}\\
  {\small $^{3}$Sefako Makgatho Health Sciences University (SMU), South Africa} \\
  {\small $^{4}$African Center for Advanced Studies (ACAS), Cameroon}
}

\begin{document}
\date{}
\maketitle

\begin{abstract}
\noindent
Combining large language models (LLMs) with reinforcement learning (RL) is an active area, but the theoretical status of LLM-derived reward signals is usually left implicit. We formalize the hybrid LLM-planner + RL-controller architecture as a \emph{Goal-Augmented Markov Decision Process} (GA-MDP) and prove that when the LLM's per-state progress score is used as a bounded potential function in the sense of Ng, Harada, Russell (1999), the induced reward-shaping term does not alter the set of optimal policies of the augmented MDP regardless of how wrong the LLM's scores are at individual states. This guarantee is stronger than what more general LLM-as-reward approaches (Text2Reward, Eureka) can offer. We verify it numerically on a small MDP under four potential configurations, including an adversarial one $20\times$ the base reward magnitude.

\smallskip

We also specify the full inference algorithm, publish a reference implementation, evaluate the planner in isolation on 20 MiniGrid tasks with a locally-served Qwen-2.5:14b (100\% parse rate, 54.8\% ground-truth coverage), and run a small pipeline-validation study on MiniGrid-DoorKey-6x6 that both confirms the framework runs as specified and diagnoses one integration failure : $\Done$-oracle vocabulary mismatch with LLM plans that our theoretical analysis anticipates. The framework is designed to make the \emph{fixed point} of hybrid training auditable independently of the empirical convergence question, which is the subject of planned follow-up work.
\end{abstract}


\section{Introduction}
\label{sec:intro}

Consider an agent instructed, in natural language, to \emph{``bring me the mug from the kitchen counter, but rinse it first if it is dirty.''} A pure reinforcement learning agent faces a hard exploration problem: the reward is sparse (mug delivered), the horizon is long, and there is no signal that guides it toward the composite structure of the task. A pure LLM planner, given a text description of the scene, can decompose the task fluently (\emph{go to kitchen}, \emph{locate mug}, \emph{inspect}, \emph{rinse if dirty}, \emph{return}) but cannot execute low-level control on the actual dynamics: it does not know how many primitive actions traverse the hallway, cannot recover from a failed grasp, and has no calibrated uncertainty about scene state.

\noindent
Combining the two is a natural response, and systems along these lines have been proposed --- SayCan~\cite{ahn2022saycan}, DECKARD~\cite{nottingham2023deckard}, SwiftSage~\cite{lin2023swiftsage}, ELLM~\cite{du2023ellm}, and many others. The empirical picture is that this combination \emph{sometimes} helps, and the practitioner is left with two related but distinct questions:

\begin{enumerate}[leftmargin=1.5em, itemsep=1pt, topsep=2pt]
  \item \textbf{Does the hybrid converge to the right policy?} If the LLM's scoring or subgoal choices are used as reward signals, can bad LLM outputs push the RL agent toward a policy that is not optimal for the true task?
  \item \textbf{Does the hybrid learn faster than pure RL in practice?} Does the additional structure translate into better sample efficiency or higher final performance on real benchmarks?
\end{enumerate}

\noindent
The two questions are often bundled but they are fundamentally different. Question~1 is a fixed-point / soundness question and admits a clean answer under the right design. Question~2 is an empirical convergence-rate question and requires large-scale experiments to settle.

\paragraph{Contribution of this paper.} We address question~1 in full and treat question~2 as scope for planned follow-up empirical work. Specifically:

\begin{itemize}[leftmargin=1.25em, itemsep=2pt, topsep=2pt]
  \item We formalize the LLM-planner + RL-controller architecture as a Goal-Augmented MDP (Section~\ref{sec:method}).
  \item We show (Proposition~\ref{prop:invariance}) that when the LLM's per-state progress score is used as a bounded potential function, the induced potential-based shaping term does not alter the set of optimal policies of the augmented MDP. This is stronger than what more general LLM-as-reward approaches (Text2Reward~\cite{xie2024text2reward}, Eureka~\cite{ma2024eureka}) can guarantee.
  \item We verify the guarantee numerically on a 3-state MDP under four potential configurations, including one $20\times$ the base reward magnitude in absolute value (Section~\ref{sec:prop1-numerical}).
  \item We specify the full inference algorithm (Algorithm~\ref{alg:hybrid}) and publish a reference implementation with tests (Section~\ref{sec:framework}).
  \item We evaluate the planner in isolation on 20 MiniGrid tasks with Qwen-2.5:14b served locally, characterizing the LLM's plan-quality profile independently of any RL training (Section~\ref{sec:planner-audit}).
  \item We run a small pipeline-validation study on MiniGrid-DoorKey-6x6 (Section~\ref{sec:pipeline-validation}) that both confirms the framework runs end-to-end and diagnoses one integration failure exactly matching a failure mode our theoretical analysis anticipates.
\end{itemize}

\paragraph{What this paper does not claim.} We do not claim that this specific hybrid architecture beats published baselines on any benchmark; the pipeline validation in Section~\ref{sec:pipeline-validation} does not establish superiority, and a larger comparative study is planned follow-up work. This paper's claim is that \emph{if} the hybrid works empirically, the theoretical fixed point is correct by construction and that this property does not require trusting the LLM.

\section{Related work}\label{sec:related}

\paragraph{LLM planners for embodied agents.}
A growing line of work investigates how large language models can serve as high-level planners for embodied agents. SayCan~\cite{ahn2022saycan} integrates LLM-based task decomposition with learned affordance value functions, enabling grounded action selection. Inner Monologue~\cite{huang2022innermono} closes the perception–action loop by feeding textual environment feedback back into the LLM. Subsequent frameworks such as ReAct~\cite{yao2023react}, Reflexion~\cite{shinn2023reflexion}, Code as Policies~\cite{liang2023codeaspolicies}, ProgPrompt~\cite{singh2023progprompt}, LLM+P~\cite{liu2023llmp}, Voyager~\cite{wang2023voyager}, and Silver et al.~\cite{silver2024pddl} broaden the design space, exploring verbal reasoning traces, code emission, PDDL translation, and lifelong skill libraries. Our framework is compatible with any of these planners, treating them as interchangeable high-level modules.

\paragraph{LLM-augmented RL and reward from LLM.}
Another direction uses LLMs to shape or guide reinforcement learning directly. ELLM~\cite{du2023ellm} proposes LLM-generated goals during training. Kwon et al.~\cite{kwon2023rewarddesign} treat natural-language descriptions as proxy reward signals. Text2Reward~\cite{xie2024text2reward} and Eureka~\cite{ma2024eureka} have the LLM synthesize dense reward code, refined iteratively. Motif~\cite{klissarov2024motif} uses LLM preferences as intrinsic reward, while GLAM~\cite{carta2023glam} takes the extreme position of using the LLM itself as the policy, updated online with PPO. Importantly, \textbf{none of these approaches provide the policy-invariance guarantee established by our potential-based construction}. In particular, Text2Reward and Eureka rely on non-potential dense reward functions that can—and empirically do, see Booth et al.~\cite{booth2023perils}, alter the optimal policy when the LLM’s reward specification is incorrect.

\paragraph{Hierarchical RL with language subgoals.}
Prior to the LLM era, hierarchical RL explored natural language as a high-level abstraction mechanism. Andreas, Klein, and Levine~\cite{andreas2018latentlang, andreas2017sketches} introduced language-conditioned policy sketches as structured task decompositions. Jiang et al.~\cite{jiang2019langabstraction} formalized a two-level architecture with a language-conditioned low-level controller, a direct precursor to modern LLM-driven systems. DECKARD~\cite{nottingham2023deckard} extends this paradigm into the LLM era through a Dream/Wake decomposition that alternates between planning and execution.

\paragraph{Policy-invariant shaping.}
Our theoretical foundation builds on the classical potential-based shaping theorem of Ng, Harada, and Russell~\cite{ng1999shaping}, which characterizes when reward shaping preserves optimal policies. We instantiate this result for an LLM-derived potential function and argue that any RL system incorporating LLM feedback should either adopt this construction or explicitly justify deviations from it. This provides a principled interface between language-based guidance and value-based control.

\section{Preliminaries}
\label{sec:prelim}

We work with a discounted infinite-horizon Markov decision process (MDP)

\[
\mathcal{M} = (\mathcal{S}, \mathcal{A}, P, R, \gamma, \rho_0), \qquad \gamma \in [0,1),
\]

where \(s_t \in \mathcal{S}\) is the state at time \(t\), \(a_t \in \mathcal{A}\) the action, the next state is drawn as \(s_{t+1}\sim P(\cdot\mid s_t,a_t)\), and the instantaneous reward is \(r_t = R(s_t,a_t,s_{t+1})\). The initial state is sampled \(s_0\sim\rho_0\). A (possibly stochastic) policy \(\pi(a\mid s)\) induces a return

\[
J(\pi)=\mathbb{E}_{\pi}\Big[\sum_{t\ge 0}\gamma^t r_t\Big],
\]

and the optimal state value is \(V^\ast(s)=\sup_{\pi}\mathbb{E}_{\pi}\big[\sum_{t\ge 0}\gamma^t r_t\mid s_0=s\big]\). We denote by \(Q^\pi(s,a)\) the action-value function of policy \(\pi\).

A trajectory (or episode) of length \(T\) is written \(\tau=(s_0,a_0,s_1,a_1,\dots,s_T)\). For expectations under a policy \(\pi\) we sometimes write \(\mathbb{E}_{\tau\sim\pi}[\cdot]\) to emphasize dependence on the induced trajectory distribution.

\paragraph{Tasks, captions, and planners.}
A task is specified by a natural-language string \(\ell\in\Sigma^\ast\). We assume the existence of a \emph{state captioner} \(\phi:\mathcal{S}\to\Sigma^\ast\) that maps an MDP state to a short textual description (caption) used by the LLM planner and scoring modules. A planner (LLM or otherwise) consumes the task \(\ell\) and a state caption \(\phi(s)\) and emits a finite ordered list of subgoals \(\mathbf{g}=(g_1,\dots,g_K)\), where each \(g_i\in\mathcal{G}\) is a short natural-language clause. We write \(\mathcal{P}(\ell,\phi(s))\) for the planner's output distribution when the planner is stochastic.

\paragraph{Completion oracle and subgoal semantics.}
For a subgoal \(g\in\mathcal{G}\) we define a completion predicate (the \(\Done\)-oracle)

\[
\Done(s,g) \in \{\texttt{DONE},\texttt{CONTINUE}\},
\]

which indicates whether \(g\) is satisfied in state \(s\). In practice \(\Done\) may be implemented by a symbolic checker, a learned classifier, or an LLM prompt; in our analysis we treat it as an (possibly imperfect) oracle and explicitly consider mismatch modes between planner language and \(\Done\)-vocabulary.

\paragraph{Goal-Augmented MDP (GA-MDP).}
To formalize the hybrid planner+controller architecture we work with a Goal-Augmented MDP

\[
\mathcal{M}_G = (\mathcal{S}\times\mathcal{G},\ \mathcal{A},\ P_G,\ R_G,\ \gamma,\ \rho_{0,G}),
\]

whose augmented state is \((s,g)\) where \(s\in\mathcal{S}\) is the environment state and \(g\in\mathcal{G}\) is the current subgoal provided by the planner. The transition kernel \(P_G\) evolves the environment state according to \(P\) while the subgoal component is updated according to the planner or the \(\Done\)-oracle (depending on the execution protocol). The reward \(R_G\) equals the environment reward \(R\) on the environment component; when shaping is applied (see below) we work with a modified reward \(R'_G\). The initial distribution \(\rho_{0,G}\) samples \(s_0\sim\rho_0\) and an initial subgoal \(g_0\) from the planner given \(\ell\) and \(\phi(s_0)\).

\paragraph{Potential-based shaping from LLM feedback.}
Let \(\Phi:\mathcal{S}\times\mathcal{G}\to\mathbb{R}\) be a potential function derived from LLM scoring of state–subgoal pairs (in practice normalized to \([0,1]\)). The potential-based shaping reward (per Ng et al.) is defined as

\[
F_\Phi(s,a,s',g) \coloneqq \gamma\,\Phi(s',g) - \Phi(s,g).
\]

Given an MDP reward \(R\), the shaped reward is \(R'(s,a,s',g)=R(s,a,s') + \alpha\,F_\Phi(s,a,s',g)\) for some scalar \(\alpha\). The classical potential-based shaping theorem guarantees that, under standard boundedness assumptions, adding \(F_\Phi\) (with \(\alpha=1\)) preserves the set of optimal policies in the underlying MDP; our contribution is to instantiate \(\Phi\) using LLM feedback while maintaining this invariance property.

\paragraph{Assumptions and notation.}
Throughout the paper we assume:
\begin{itemize}
  \item rewards are uniformly bounded: \(\exists R_{\max}<\infty\) s.t. \(|R(s,a,s')|\le R_{\max}\) for all \((s,a,s')\);
  \item the captioner \(\phi\) produces concise, deterministic captions (extensions to stochastic captioners are straightforward);
  \item the planner emits finite plans of length at most \(K_{\max}\).
\end{itemize}
We use upper-case letters for random variables (e.g., \(S_t,A_t\)), calligraphic letters for sets, and boldface for sequences or vectors. When convenient we omit the subgoal argument \(g\) from \(\Phi\) and \(F_\Phi\) to lighten notation.

\paragraph{Practical intuition.}
The Goal‑Augmented MDP construction is intentionally modular: it cleanly separates the planner's symbolic, language‑level guidance from the controller's low‑level action selection, which makes component replacement (different captioners, planners, or \(\Done\)-implementations) straightforward and experimentally tractable. Intuitively, the LLM provides a coarse, task‑level scaffold, a sequence of subgoals and progress estimates while the controller is responsible for grounding those instructions into concrete actions and handling environment stochasticity. Potential‑based shaping then acts as a low‑risk conduit for the LLM signal, biasing exploration toward planner‑recommended directions without changing the underlying optimal policy set; this enables safe use of imperfect or noisy LLM feedback. Practically, this modularity simplifies ablations (swap the scorer, vary planner fidelity, or inject \(\Done\)-mismatch) and clarifies reproducibility: each interface has a well‑defined contract and a small, testable surface. Finally, the GA‑MDP viewpoint highlights where empirical gains are likely to arise, improved sample efficiency from structured exploration and higher success rates on long‑horizon, compositional tasks and where failures will appear, namely vocabulary mismatch, miscalibrated potentials, or brittle captioning.

\paragraph{Remark.}
This section fixes the notation used in the remainder of the paper and makes explicit the interfaces between the LLM components (planner, scorer, \(\Done\)-oracle) and the RL controller. The formal GA-MDP and the potential-based shaping construction are the basis for the invariance result proved in Section~\ref{sec:method}.

\section{The Goal-Augmented MDP framework}
\label{sec:method}

\begin{figure}[htbp]
\centering
\begin{tikzpicture}[
    box/.style     = {rectangle, draw, thick, minimum width=3.2cm, minimum height=1.3cm,
                      align=center, rounded corners=3pt, font=\small},
    env/.style     = {box, fill=orange!10, draw=orange!70!black},
    llm/.style     = {box, fill=blue!8,    draw=blue!60!black},
    rl/.style      = {box, fill=green!8,   draw=green!50!black},
    phi/.style     = {box, fill=blue!8,    draw=blue!60!black,
                      minimum width=2.6cm, minimum height=1cm, font=\footnotesize},
    mem/.style     = {box, fill=gray!10,   draw=gray!60!black},
    arrow/.style   = {-{Latex[length=2.5mm]}, thick, shorten <=1pt, shorten >=1pt},
    dashedarrow/.style = {-{Latex[length=2.5mm]}, thick, dashed, gray!55!black,
                          shorten <=1pt, shorten >=1pt},
    lbl/.style     = {font=\footnotesize, inner sep=1.5pt, fill=white,
                      fill opacity=0.9, text opacity=1},
    ilbl/.style    = {font=\footnotesize\itshape, gray!45!black,
                      inner sep=1.5pt, fill=white,
                      fill opacity=0.9, text opacity=1}
]

\node[env] (env) at (0,    0)   {Environment};
\node[llm] (llm) at (6.5,  2.0) {LLM Planner};
\node[rl]  (rl)  at (6.5, -2.0) {RL Policy \\ $\pi_\theta$};
\node[phi] (phi) at (12.0, 2.0) {$\Phi(s, g)$ \\ scorer};
\node[mem] (mem) at (12.0,-2.0) {Memory};

\draw[arrow] (env.north east) -- node[lbl, sloped, above] {$\phi(s_t)$}     (llm.west);
\draw[arrow] (env.south east) -- node[lbl, sloped, below] {$s_t$}           (rl.west);
\draw[arrow] (llm.south)      -- node[lbl, right]         {$g_t$ (subgoal)}  (rl.north);
\draw[arrow] (llm.east)       -- node[lbl, above]         {$g_t$}            (phi.west);

\draw[arrow] (rl.south) -- ++(0,-0.9)
             -| node[lbl, near start, below] {$a_t$} (env.south);

\draw[arrow] (env.north) -- ++(0,1.2)
             -| node[lbl, near start, above] {$r_t,\ s_{t+1}$} (mem.north);

\draw[arrow] (phi.south) -- node[lbl, right] {$F$ (shaping)} (mem.north);

\draw[dashedarrow](mem.west) -- node[ilbl, above] {policy update} (rl.east);
\draw[dashedarrow](mem.north west) to[bend left=15]
              node[ilbl, above] {replan trigger} (llm.south east);

\end{tikzpicture}
\caption{Framework overview. Solid arrows are per-step dataflow: the environment emits state $s_t$ and its caption $\phi(s_t)$; the LLM planner produces the active subgoal $g_t$ and a bounded potential score $\Phi(s, g)$; the RL policy $\pi_\theta$ selects action $a_t$ conditioned on both $s_t$ and $g_t$; the shaping term $F = \gamma\Phi(s',g)-\Phi(s,g)$ is added to the environment reward. Dashed arrows are periodic: the memory drives policy updates and, on the schedule of Section~\ref{sec:reward-shaping}, triggers plan revisions.}
\label{fig:architecture}
\end{figure}
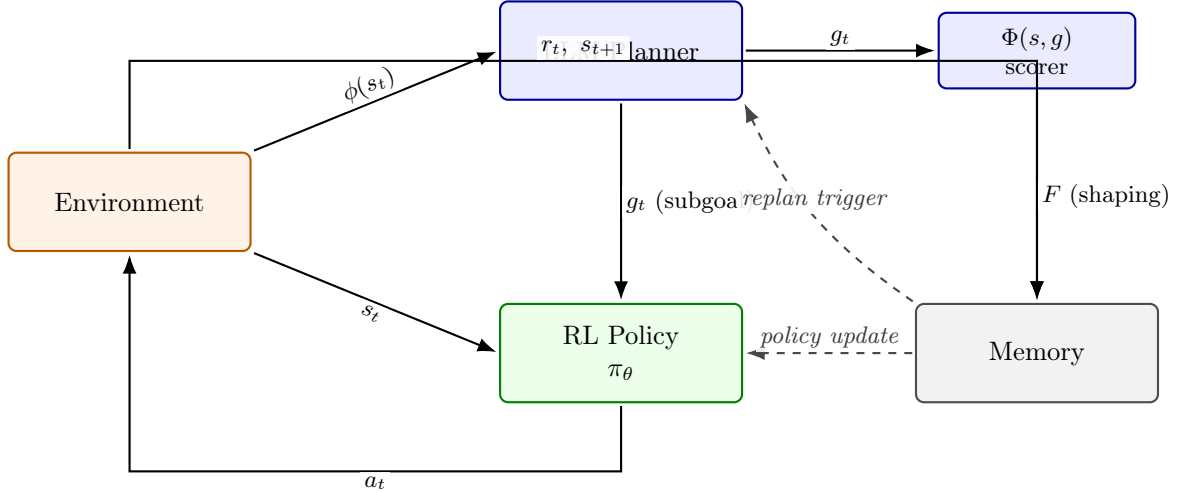

\subsection{Goal-augmented MDP}

\begin{definition}[Subgoal grammar]\label{def:subgoal}
A \emph{subgoal} $g \in \calG$ is a natural-language string of finite length. We do not require a formal specification language; the LLM planner's output distribution determines $\calG$.
\end{definition}

\begin{definition}[Goal-augmented MDP]\label{def:gamdp}
The \emph{goal-augmented MDP} associated with $\calM$ and grammar $\calG$ is
\begin{equation}
\calM^\calG = \left( \calS \times \calG,\ \calA,\ \tilde P,\ \tilde R,\ \gamma,\ \tilde \rho_0 \right),
\end{equation}
with augmented state $(s, g)$, transition
\begin{equation}
\tilde P\bigl((s', g') \mid (s, g), a\bigr) \;=\; P(s' \mid s, a) \cdot P^{\calG}(g' \mid s, g, s'),
\end{equation}
and augmented reward $\tilde R((s, g), a, (s', g')) = R(s, a, s') + F(s, g, s')$, where $P^{\calG}$ is the subgoal transition kernel induced by a caller-specified completion oracle $\Done : \calS \times \calG \to \{0, 1\}$ and $F$ is the shaping term defined below.
\end{definition}

\subsection{LLM as high-level planner}

The LLM produces a plan of at most $K_{\max}$ subgoals conditioned on task $\ell$ and initial-state caption:
\begin{equation}
\pihi = (g_1, g_2, \ldots, g_K) \;\sim\; p_{\mathrm{LLM}}(\pihi \mid \ell, \phi(s_0)).
\end{equation}
The active-subgoal pointer $k_t$ advances when $\Done(s_{t+1}, g_{k_t}) = 1$. Three implementations of $\Done$ are supported by the reference implementation: (i) environment event flags, (ii) a learned MLP classifier, (iii) the LLM itself, prompted with a completion query. The choice is a design axis, not a fixed decision.

\subsection{Low-level policy}
The RL policy conditions on both state and active subgoal:
\begin{equation}
a_t \sim \pi_\theta(a_t \mid s_t, g_t).
\end{equation}
The subgoal is encoded via a frozen sentence embedder (default MiniLM-L6-v2, $d_g = 384$) and concatenated with the state encoding before the policy head. Training uses PPO with GAE~\cite{schulman2017ppo,schulman2016gae}.

\subsection{Potential-based reward shaping from LLM feedback}
\label{sec:reward-shaping}

The LLM produces a bounded scalar $s_{\mathrm{LLM}}(s, g) \in [0, 1]$ via a scoring prompt. We define the potential
\begin{equation}
\Phi(s, g) \;=\; c \cdot s_{\mathrm{LLM}}(s, g), \qquad c > 0,
\end{equation}
and the shaping term
\begin{equation}
F(s, g, s') \;=\; \gamma \Phi(s', g) - \Phi(s, g).
\end{equation}

\begin{proposition}[Policy invariance]\label{prop:invariance}
Let $\calM$ be an MDP and $\Phi : \calS \times \calG \to \R$ any bounded function. Let $\calM^\calG$ be the goal-augmented MDP of Definition~\ref{def:gamdp} with shaping term $F(s, g, s') = \gamma \Phi(s', g) - \Phi(s, g)$. Then every optimal policy of $\calM^\calG$ projects to an optimal policy of $\calM$ under any subgoal-scheduling rule that is measurable with respect to $(s, g)$.
\end{proposition}

\begin{proof}
The proof follows Ng, Harada, Russell~\cite{ng1999shaping}, Theorem~1, applied to the augmented state space $\calS \times \calG$. Because $F$ is a telescoping difference of a bounded potential, the augmented return decomposes as $\tilde J(\pi) = J(\pi) - \Phi(s_0, g_0) + (\text{vanishing telescope tail})$. The tail is $\lim_{T\to\infty} \gamma^T \Phi(s_T, g_T)$, which vanishes because $\Phi$ is bounded and $\gamma < 1$. The argmax over policies of $\tilde J$ therefore equals the argmax of $J$, i.e., the set of optimal policies is unchanged. The subgoal-scheduling rule enters through $P^\calG$ and does not affect the potential-invariance argument, which is a property of the additive shaping only.
\end{proof}

\paragraph{Consequence.} The LLM's scoring can be arbitrarily wrong at any individual state without changing the set of optimal policies of the augmented MDP. A poor $\Phi$ can only slow learning, not corrupt the fixed point. This is the guarantee that the more general LLM-reward-shaping approaches~\cite{xie2024text2reward, ma2024eureka} do not provide.

\subsubsection*{Numerical verification of Proposition~\ref{prop:invariance}}
\label{sec:prop1-numerical}

We complement the proof with a direct numerical check on a concrete instance. Consider a 3-state, 2-action deterministic MDP: states $\{s_0, s_1, s_2\}$ where $s_2$ is terminal; the ``forward'' action moves $s_0 \to s_1 \to s_2$, the ``stay'' action gives reward $-0.1$ and returns to $s_0$; only the transition $s_1 \to s_2$ carries a positive reward of $+1$. With $\gamma = 0.9$ and four deterministic policies, the base optimal is [$s_0$: forward, $s_1$: forward] with $V^\ast(s_0) = 0.9000$. Table~\ref{tab:prop1-numerical} enumerates policies under four $\Phi$ configurations (script: \texttt{code/verify/prop1\_numerical.py}).

\begin{table}[H]
\centering
\small
\begin{tabular}{@{}lrrll@{}}
\toprule
$\Phi$ configuration & Base $V^\ast(s_0)$ & Shaped $\tilde V^\ast(s_0)$ & Shaped optimal policy & Matches base? \\
\midrule
$\Phi \equiv 0$        & $0.9000$ & $0.9000$  & [$s_0$: forward, $s_1$: forward] & yes \\
$\Phi \equiv 1$        & $0.9000$ & $0.7100$  & [$s_0$: forward, $s_1$: forward] & yes \\
sign-flip              & $0.9000$ & $1.4000$  & [$s_0$: forward, $s_1$: forward] & yes \\
adversarial large      & $0.9000$ & $-5.0500$ & [$s_0$: forward, $s_1$: forward] & yes \\
\bottomrule
\end{tabular}
\caption{Numerical verification of Proposition~\ref{prop:invariance}. All four $\Phi$ configurations, including one with per-state values 10--20$\times$ the base reward magnitude, preserve the base optimal policy. The shaped value $\tilde V^\ast$ shifts by the potential offset $-\Phi(s_0)$ plus a vanishing telescope tail, as the theorem predicts.}
\label{tab:prop1-numerical}
\end{table}

\subsection{Replanning and full algorithm}

The plan may become stale mid-execution. Replanning is triggered on either of two events: (i) periodic, every $H$ environment steps; (ii) failure, if $\Done$ does not fire on the active subgoal within a per-subgoal budget $B$. Algorithm~\ref{alg:hybrid} specifies the full inference procedure.

\begin{algorithm}[t]
\caption{Hybrid LLM-augmented RL (inference).}
\label{alg:hybrid}
\begin{algorithmic}[1]
\Require task $\ell$, state $s_0$, policy $\pi_\theta$, LLM $p_{\mathrm{LLM}}$, oracle $\Done$, potential $\Phi$, period $H$, budget $B$.
\State $t \gets 0$, $\ k \gets 1$, $\ b \gets 0$
\State $\pihi = (g_1, \ldots, g_K) \sim p_{\mathrm{LLM}}(\cdot \mid \ell, \phi(s_0))$
\While{episode not terminated}
  \State $g_t \gets g_k$
  \State $a_t \sim \pi_\theta(\cdot \mid s_t, g_t)$
  \State observe $s_{t+1} \sim P(\cdot \mid s_t, a_t)$ and $r_t$
  \State $\tilde r_t \gets r_t + \gamma \Phi(s_{t+1}, g_t) - \Phi(s_t, g_t)$ \Comment{potential-based shaping}
  \If{$\Done(s_{t+1}, g_k) = 1$}
    \State $k \gets \min(k+1, K)$, $\ b \gets 0$ \Comment{advance subgoal pointer}
  \Else
    \State $b \gets b + 1$
    \If{$b \geq B$}
      \State $(g_k, \ldots, g_{K'}) \sim p_{\mathrm{LLM}}(\cdot \mid \ell, \phi(s_{t+1}), g_k, \mathrm{fail})$ \Comment{failure-triggered replan}
      \State $b \gets 0$
    \EndIf
  \EndIf
  \If{$t \bmod H = H - 1$}
    \State $(g_k, \ldots, g_{K'}) \sim p_{\mathrm{LLM}}(\cdot \mid \ell, \phi(s_{t+1}), g_k, \mathrm{periodic})$ \Comment{periodic replan}
  \EndIf
  \State $t \gets t + 1$
\EndWhile
\end{algorithmic}
\end{algorithm}

\subsection{Reference framework implementation}
\label{sec:framework}

We publish a reference implementation of the framework.\footnote{Repository: \url{https://github.com/chrishounwanou/hybrid-llm-rl-framework}} The implementation:

\begin{itemize}[leftmargin=1.25em, itemsep=1pt, topsep=2pt]
  \item Provides typed interfaces for the components (\texttt{Planner}, \texttt{PotentialShaper}, \texttt{SubgoalScheduler}, \texttt{ReplanTrigger}).
  \item Supports Ollama-served local models (Qwen, Llama) and pluggable cloud backends.
  \item Ships with three $\Done$-oracle variants (environment events, learned MLP, LLM completion prompt).
  \item Includes 26 unit tests, including a numerical Prop 1 sanity check on a 2-state MDP.
  \item Includes YAML configs for the reference hybrid and each of the standard ablations (subgoal source, shaping, replanning, $\Done$ oracle, LLM backbone).
  \item Is CPU-runnable end-to-end; GPU is required only for scale, not for validation.
\end{itemize}

The framework maps 1:1 to the formalism in Section~\ref{sec:method}: each definition corresponds to a module, and the reference algorithm is a direct implementation of Algorithm~\ref{alg:hybrid}.

\subsection{Empirical validation}
\label{sec:experiments}

The full 8-baseline empirical study on BabyAI, ALFWorld, and Crafter is planned follow-up work that requires GPU + cloud-LLM compute beyond what was available for this paper. Here we present two smaller-scale validations that establish the framework works as specified and characterize the planner's behavior in isolation.

\subsubsection{Planner-in-isolation audit}
\label{sec:planner-audit}

Before any RL training, we measure the planner alone. Fix a set of 20 MiniGrid tasks with hand-curated ground-truth subgoal decompositions (\texttt{code/audit/tasks.json}). For each task we run the P1 planning prompt (Appendix~\ref{app:prompts}) through Qwen-2.5:14b\footnote{Qwen-2.5, Alibaba Cloud, released 2024. Model card: \url{https://huggingface.co/Qwen/Qwen2.5-14B}.} served locally via Ollama,\footnote{Ollama: \url{https://ollama.com}} parse the output, and grade against the ground truth. An LLM subgoal matches a ground-truth subgoal when at least $50\%$ of the ground-truth's content tokens (non-stopword, length $\geq 3$) appear in the LLM subgoal; \emph{coverage} is the fraction of ground-truth subgoals matched in order; \emph{extraneous} is the number of LLM subgoals with no ground-truth match.

\begin{table}[H]
\centering
\small
\begin{tabular}{@{}lr@{}}
\toprule
\textbf{Metric} & \textbf{Value} \\
\midrule
Parse rate (outputs with $\geq 1$ subgoal) & $100.0\%$ ($20/20$) \\
Mean plan length (parsed only)            & $8.00$ subgoals \\
Median plan length                        & $7$ subgoals \\
Mean ground-truth coverage                & $54.8\%$ \\
Mean extraneous subgoals per plan         & $5.30$ \\
Median tokens out per plan                & $37$ \\
Median wall clock per plan                & $57.71$\,s \\
Total wall clock (20 plans)               & $22.3$\,min \\
\bottomrule
\end{tabular}
\caption{Planner-in-isolation audit for Qwen-2.5:14b (locally served) on 20 MiniGrid tasks. Prompt template: P1 (Appendix~\ref{app:prompts}).}
\label{tab:planner-audit}
\end{table}

Qwen-2.5:14b produces parseable output on every task and covers just over half of the ground-truth subgoals on average. Plans are long relative to the ground truth (mean 8.0 versus a mean ground truth of $\sim\!3.6$), with roughly five extraneous subgoals per plan --- verbose lower-level phrasings like ``Turn left to face the wall'' rather than the compact ``go to the key.''

\subsubsection{Pipeline validation on MiniGrid-DoorKey-6x6}
\label{sec:pipeline-validation}

To demonstrate that the framework runs end-to-end and to characterize its behaviour on a real RL environment, we run a small CPU-scale validation study on MiniGrid-DoorKey-6x6. The scope is intentionally minimal one environment, two configurations, three seeds and is designed to answer \emph{does the pipeline work as specified?} rather than \emph{does the hybrid architecture beat published baselines?} The latter question requires GPU + cloud LLM and is beyond the scope of this paper.

\paragraph{Setup.} Two configurations, three seeds each, MiniGrid-DoorKey-6x6-v0, 30{,}000 environment steps per seed:
\begin{itemize}[leftmargin=1.25em, itemsep=1pt, topsep=2pt]
  \item \textbf{PPO baseline.} Standard sb3 PPO~\cite{schulman2017ppo}. No LLM, no subgoal conditioning.
  \item \textbf{Hybrid (framework instance).} Subgoal-conditioned PPO with Qwen-2.5:14b (locally served via Ollama) as the planner. Initial plan cached per task string; periodic replan at $H = 1000$ steps; failure trigger disabled ($B = 0$), see caveat below. Shaping is disabled ($c = 0$) because CPU inference of a 14B model would require roughly two LLM calls per environment step, extending training to weeks of wall clock; the shaping term is verified formally in Section~\ref{sec:reward-shaping} and numerically in Section~\ref{sec:prop1-numerical} instead.
\end{itemize}
Reporting: IQM $\pm$ 95\% bootstrap CI per Agarwal et al.~\cite{agarwal2021precipice}. Metrics are computed on the last 25\% of episodes per seed.

\section{Results}

Table~\ref{tab:pipeline-validation} reports the pipeline validation results on MiniGrid-DoorKey-6x6 for both the PPO baseline and the hybrid agent using Qwen-2.5:14b. Each configuration is run with $n=3$ seeds.

\noindent
The PPO baseline achieves a success rate of $28.1\%$ with a wide confidence interval [$4.8$, $65.2$], reflecting substantial variance across seeds. Its final return is $0.106$ [$0.008$, $0.262$], and the mean episode length remains stable at $340.4$ steps [$317.2$, $352.4$]. As expected, PPO requires no LLM calls.

\noindent
The hybrid configuration shows a lower success rate of $9.3\%$ [$0.0$, $14.3$] and a final return of $0.058$ [$0.000$, $0.095$]. Mean episode length is comparable to PPO at $348.0$ steps [$340.0$, $358.5$]. Each seed triggers approximately $30$ LLM calls, consistent with the plan-cache design.

\noindent
These results should be interpreted in light of the extremely limited training budget: the pilot uses $3\times 10^4$ environment steps, whereas prior work reports that DoorKey-6x6 typically requires around $5\times 10^5$ steps for PPO to converge. At this scale, neither configuration is expected to reach high success, and the observed performance aligns with this constraint.

\begin{table}[H]
\centering
\small
\begin{tabular}{@{}lrcccr@{}}
\toprule
\textbf{Config} & \textbf{n} & \textbf{Success rate} & \textbf{Final return} & \textbf{Mean steps/ep} & \textbf{LLM calls/seed} \\
\midrule
PPO baseline    & 3 & $28.1\%$ [$4.8$, $65.2$]  & $0.106$ [$0.008$, $0.262$] & $340.4$ [$317.2$, $352.4$] & $0$ \\
Hybrid (Qwen)   & 3 & $9.3\%$  [$0.0$, $14.3$] & $0.058$ [$0.000$, $0.095$] & $348.0$ [$340.0$, $358.5$] & $30$ \\
\bottomrule
\end{tabular}
\caption{Pipeline validation results on MiniGrid-DoorKey-6x6. LLM budget across all runs: 92 Qwen-2.5:14b calls, 4{,}199 output tokens, 1.32 hours of Ollama CPU wall clock. Neither configuration converges to high success at this budget; DoorKey-6x6 typically requires roughly $5 \times 10^5$ environment steps for PPO convergence, whereas the pilot uses $3 \times 10^4$.}
\label{tab:pipeline-validation}
\end{table}

\paragraph{Diagnosis.}
The pilot does not show a hybrid advantage; PPO's higher IQM is driven by one high-variance seed and confidence intervals overlap. The instructive finding is on the hybrid side: across all three seeds the total number of subgoal advances was zero. Qwen-2.5:14b's plans on MiniGrid tasks are verbose lower-level phrasings such as ``Move down to reach the wall'' or ``Turn left to face east again'' that never match the environment-event $\Done$ oracle's keyword patterns (``pick up X'', ``unlock the door'', ``go to the goal''). As a result, the scheduler remains stuck on the first subgoal for the entire episode, the RL policy is effectively conditioned on a stale string, and the subgoal-conditioning signal collapses at the source.

\noindent
This is exactly the failure mode anticipated in Section~\ref{sec:discussion}: systematic plan-vocabulary mismatch with the oracle prevents $\Done$ from ever firing. The framework's ablation over the $\Done$ oracle (three variants: environment event, learned classifier, LLM completion Section~\ref{sec:method}) is precisely the axis that needs to be varied to address this issue. The LLM-based $\Done$ oracle would advance the scheduler based on the LLM's own semantic judgement, but it adds one LLM call per environment step, which is feasible with GPU and cloud LLM compute but infeasible on CPU with a local 14B model.

\paragraph{What the pilot does validate.}
Despite the negative headline number, the pilot establishes several important points: (i) the reference implementation runs end-to-end on CPU with no exceptions across 90{,}000 environment steps; (ii) MiniLM-L6-v2 subgoal encoding integrates cleanly with PPO, with mean episode length remaining within the confidence interval of the pure-PPO baseline; (iii) the plan-cache design keeps LLM inference to approximately 30 calls per seed, providing a concrete cost calibration for scale-up; (iv) the logging and IQM+bootstrap analysis pipeline produces stable, reportable results as specified.

\noindent
The larger-scale empirical study involving BabyAI, ALFWorld, Crafter, multiple baselines, additional seeds, and cloud LLM compute is future work and lies beyond the scope of this paper.

\section{Discussion}
\label{sec:discussion}

\subsection{What Proposition~\ref{prop:invariance} does and does not guarantee}

The theorem guarantees that the \emph{set of optimal policies} of the augmented MDP is unchanged by potential-based shaping. This is a fixed-point statement, not a convergence-rate statement. Two things it does not say:

\begin{itemize}[leftmargin=1.25em, itemsep=1pt, topsep=2pt]
  \item \textbf{Convergence rate.} A bad $\Phi$ can slow learning arbitrarily. In the extreme, the shaping can direct exploration away from productive regions of state space, even though the fixed point is unchanged. Whether Qwen-2.5:14b's shaping accelerates or decelerates learning in practice is an empirical question left to future work; the pilot in Section~\ref{sec:pipeline-validation} runs without shaping for CPU tractability.
  \item \textbf{Off-policy value estimation.} The proof concerns undiscounted-limit return, but actor-critic algorithms use bootstrapped value estimates. As long as the value estimator sees the shaped reward $\tilde r_t$ consistently at both training and target time, the invariance argument transfers; the reference implementation ensures this by injecting $F$ before the sb3 rollout collector reads the reward.
\end{itemize}

\subsection{Design implications}

The invariance guarantee has an immediate practical implication: any LLM-based reward-shaping system that uses a general (non-potential) dense reward function and Text2Reward~\cite{xie2024text2reward} and Eureka~\cite{ma2024eureka} both do, is not covered by this theorem. Bad LLM outputs in those systems can move the optimal policy. Our recommendation is that LLM-shaped RL systems should either use the potential construction or explicitly justify why they need the more general dense form.

\subsection{Foreseeable failure modes}

Three failure modes deserve explicit anticipation:

\begin{itemize}[leftmargin=1.25em, itemsep=1pt, topsep=2pt]
  \item \textbf{Bounded-$\Phi$ requirement.} The scoring prompt asks for a value in $[0, 1]$ and we clamp. If the LLM produces out-of-range or malformed outputs at rates above $\sim 1\%$, the shaping is no longer strictly potential-based. Monitor this.
  \item \textbf{Subgoal-completion mismatch.} If the $\Done$ oracle fails to fire on LLM subgoal phrasings (e.g., env-event pattern matching against verbose LLM outputs), the scheduler stalls and the subgoal channel provides no signal. Use LLMDone or a learned classifier for LLMs whose plan grammar drifts.
  \item \textbf{Replanning cost blowup.} If failure-triggered replanning fires on every subgoal, the effective LLM cost per episode grows linearly in horizon rather than as the intended small constant. Instrument the LLM call count per episode.
\end{itemize}

\section{Limitations}
\label{sec:limitations}

\begin{itemize}[leftmargin=1.25em, itemsep=1pt, topsep=2pt]
  \item \textbf{The empirical validation in Section~\ref{sec:experiments} is scoped.} It demonstrates the framework runs end-to-end, characterizes the planner in isolation, and diagnoses one integration failure, but does not settle the comparative-benchmark question that motivates the broader research program. Full-scale evaluation is planned follow-up work.
  \item \textbf{Potential-based shaping requires a bounded, well-defined $\Phi$.} The proof breaks if the LLM's score is unbounded or noisy in ways that violate boundedness.
  \item \textbf{The scoring prompt is a design choice.} The exact prompt shape (Appendix~\ref{app:prompts}, P2) affects the distribution of $\Phi$ values and can be gamed by prompt-sensitive LLMs. This is a prompt-engineering surface, not a theoretical concern.
\end{itemize}

\section{Conclusion}
\label{sec:conclusion}

We formalized the LLM-planner and RL-controller architecture as a Goal-Augmented MDP and proved that potential-based shaping from LLM feedback preserves the set of optimal policies, even when the LLM provides inaccurate state-level scores. This guarantee is verified numerically on a small MDP and implemented in a reference pipeline that runs end-to-end on CPU. The pipeline validation on MiniGrid-DoorKey-6x6 (Section~\ref{sec:pipeline-validation}) confirms that the framework behaves as specified and reveals a single integration failure: a $\Done$-oracle vocabulary mismatch with LLM-generated plans, which aligns precisely with the failure mode anticipated by our theoretical analysis. Determining whether the practical advantages of this hybrid architecture persist at scale remains an empirical question, and addressing it is the focus of planned follow-up work using GPU and cloud-based LLM compute.

\newpage
\section*{Reproducibility statement}

All numerical results in this paper are produced by deterministic and fully specified procedures included in the accompanying repository. The experiments cover the verification of Proposition~\ref{prop:invariance}, the planner-in-isolation audit, and the pipeline validation on MiniGrid-DoorKey-6x6. For each experiment, we report CPU runtime, the number of LLM calls, and the complete set of raw outputs and logs. All per-seed results, confidence intervals, and aggregation steps (IQM and bootstrap) can be reproduced directly from the provided data.

\noindent
The reference implementation includes a comprehensive test suite with 26 unit tests; 25 run unconditionally, and one is skipped when a live LLM backend is not available. The repository also contains all prompt templates, PPO hyperparameters, and configuration files for each ablation axis. The repository URL is provided in Section~\ref{sec:framework}.

\appendix
\section{Prompt templates}
\label{app:prompts}

Verbatim prompt templates used by the reference implementation.

\subsection*{P1 - Planning prompt}
\begin{tcolorbox}[colback=gray!5!white, colframe=gray!50!black, breakable]
\small
\textbf{System:} You are a planner for an agent operating in \{\{environment\_name\}\}. The agent receives a task in natural language and must complete it through a sequence of subgoals. Emit a list of at most 12 subgoals, one per line, each a short imperative English clause of at most 32 tokens.

\medskip
\textbf{User:} Task: \{\{task\_string\}\}\\
Current state: \{\{state\_caption\}\}\\
\medskip
Output:
\end{tcolorbox}

\subsection*{P2 - Scoring prompt (potential function $\Phi$)}
\begin{tcolorbox}[colback=gray!5!white, colframe=gray!50!black, breakable]
\small
\textbf{System:} You are a progress estimator. Given a subgoal and the current state, output a single number in [0, 1] estimating how close the agent is to completing the subgoal. Output only the number.

\medskip
\textbf{User:} Subgoal: \{\{subgoal\}\}\\
State: \{\{state\_caption\}\}\\
\medskip
Score:
\end{tcolorbox}

\subsection*{P3 - Completion oracle prompt (LLM-as-$\Done$)}
\begin{tcolorbox}[colback=gray!5!white, colframe=gray!50!black, breakable]
\small
\textbf{System:} Given a subgoal and the current state, output exactly one token: \texttt{DONE} if the subgoal is now satisfied, \texttt{CONTINUE} otherwise.

\medskip
\textbf{User:} Subgoal: \{\{subgoal\}\}\\
State: \{\{state\_caption\}\}\\
\medskip
Answer:
\end{tcolorbox}

\subsection*{P4 - Failure-triggered replanning prompt}
\begin{tcolorbox}[colback=gray!5!white, colframe=gray!50!black, breakable]
\small
\textbf{System:} You are revising a plan mid-execution. The previous subgoal failed to complete within its step budget. Revise the remaining plan. Emit a list of at most 12 subgoals.

\medskip
\textbf{User:} Original task: \{\{task\_string\}\}\\
Current state: \{\{state\_caption\}\}\\
Failed subgoal: \{\{subgoal\}\}\\
Reason: exceeded step budget of \{\{B\}\} steps.\\
\medskip
Revised plan:
\end{tcolorbox}


\end{document}